\documentclass[runningheads,a4paper]{llncs}

\usepackage{amssymb}
\usepackage{graphicx}
\usepackage{amsmath}

\newtheorem{thm}{Theorem}

\newcommand{\comments}[1]{}
\comments{Due to double-blind reviewing, there are a few comments that should be uncommented for the final version.}

\usepackage{url}
\urldef{\mailsa}\path|jose.m.pena@liu.se|  
\newcommand{\keywords}[1]{\par\addvspace\baselineskip
\noindent\keywordname\enspace\ignorespaces#1}

\begin{document}

\mainmatter  

\title{Approximate Counting of Graphical Models Via MCMC Revisited}

\titlerunning{Approximate Counting of Graphical Models Via MCMC Revisited}

%
%
\author{Jose M. Pe\~{n}a}
\authorrunning{Jose M. Pe\~{n}a}

\institute{ADIT, IDA, Link\"oping University, Sweden\\
\mailsa\\
}

%
%

\toctitle{}
\tocauthor{}
\maketitle

\begin{abstract}
In \cite{Penna2007}, MCMC sampling is applied to approximately calculate the ratio of essential graphs (EGs) to directed acyclic graphs (DAGs) for up to 20 nodes. In the present paper, we extend that work from 20 to 31 nodes. We also extend that work by computing the approximate ratio of connected EGs to connected DAGs, of connected EGs to EGs, and of connected DAGs to DAGs. Furthermore, we prove that the latter ratio is asymptotically 1. We also discuss the implications of these results for learning DAGs from data.
\end{abstract}

\section{Introduction}

Probably the most common approach to learning directed acyclic graph
(DAG) models\footnote{All the graphs considered in this paper are labeled graphs.} from data, also known as Bayesian network models, is that of
performing a search in the space of either DAGs or DAG models. In
the latter case, DAG models are typically represented as essential
graphs (EGs). Knowing the ratio of EGs to DAGs for a given number of nodes is a valuable
piece of information when deciding which space to search. For
instance, if the ratio is low, then one may prefer to search the
space of EGs rather than the space of DAGs, though the latter
is usually considered easier to traverse. Unfortunately, while the
number of DAGs can be computed without
enumerating them all \cite[Equation 8]{Robinson1977}, the only method for counting EGs that we are aware of is enumeration. Specifically, Gillispie and Perlman enumerated all the EGs for up
to 10 nodes by means of a computer
program \cite{GillispieandPerlman2002}. They showed that the ratio is around 0.27 for 7-10 nodes. They also conjectured a similar ratio for more than 10 nodes by
extrapolating the exact ratios for up to 10 nodes.

Enumerating EGs for more than 10 nodes seems challenging: To enumerate all the EGs over 10 nodes, the
computer program of \cite{GillispieandPerlman2002} needed 2253 hours
in a "mid-1990s-era, midrange minicomputer". We obviously prefer to know the exact ratio of EGs to DAGs
for a given number of nodes rather than an approximation to it.
However, an approximate ratio may be easier to obtain and serve as
well as the exact one to decide which space to search. In \cite{Penna2007}, a Markov chain Monte Carlo (MCMC) approach was proposed to
approximately calculate the ratio while avoiding enumerating EGs.
This approach consisted of the following steps. First, the author constructed a
Markov chain (MC) whose stationary distribution was uniform over the
space of EGs for the given number of nodes. Then, the author sampled that
stationary distribution and computed the ratio $R$ of essential DAGs
(EDAGs) to EGs in the sample. Finally, the author transformed this approximate
ratio into the desired approximate ratio of EGs to DAGs as follows: Since $\frac{\#EGs}{\#DAGs}$ can be expressed as $\frac{\#EDAGs}{\#DAGs}
\frac{\#EGs}{\#EDAGs}$,\footnote{We use the symbol $\#$ followed by a class of graphs to denote the cardinality of the class.} then we can approximate it by
$\frac{\#EDAGs}{\#DAGs} \frac{1}{R}$ where $\#DAGs$ and $\#EDAGs$
can be computed via \cite[Equation 8]{Robinson1977} and \cite[p. 270]{Steinsky2003}, respectively. The author reported
the so-obtained approximate ratio for up to 20 nodes. The
approximate ratios agreed well with the exact ones available in the
literature and suggested that the exact ratios are not very low (the approximate ratios were 0.26-0.27 for 7-20 nodes). This indicates that one should not expect more than a moderate gain in efficiency when searching the space of EGs instead of the space of DAGs. Of course, this is a bit of a bold claim since the gain is dictated by the average ratio over the EGs visited during the search and not by the average ratio over all the EGs in the search space. For instance, the gain is not the same if we visit the empty EG, whose ratio is 1, or the complete EG, whose ratio is $1/n!$ for $n$ nodes. Unfortunately, it is impossible to know beforehand which EGs will be visited during the search. Therefore, the best we can do is to draw (bold) conclusions based on the average ratio over all the EGs in the search space.

In this paper, we extend the work in \cite{Penna2007} from 20 to 31 nodes. We also extend that work by reporting some new approximate ratios. Specifically, we report the approximate ratio of connected EGs (CEGs) to connected DAGs (CDAGs), of CEGs to EGs, and of CDAGs to DAGs. We elaborate later on why these ratios are of interest. The approximate ratio of CEGs to CDAGs is computed from the sample as follows. First, we compute the ratio $R'$ of EDAGs to CEGs in the sample. Second, we transform this approximate
ratio into the desired approximate ratio of CEGs to CDAGs as follows: Since $\frac{\#CEGs}{\#CDAGs}$ can be expressed as $\frac{\#EDAGs}{\#CDAGs}
\frac{\#CEGs}{\#EDAGs}$, then we can approximate it by
$\frac{\#EDAGs}{\#CDAGs} \frac{1}{R'}$ where $\#EDAGs$ can be computed by \cite[p. 270]{Steinsky2003} and $\#CDAGs$ can be computed as shown in Appendix A. The approximate ratio of CEGs to EGs is computed directly from the sample. The approximate ratio of CDAGs to DAGs is computed with the help of Appendix A and \cite[Equation 8]{Robinson1977}.

The computer program implementing the MCMC approach described above is essentially the same as in \cite{Penna2007} (it has only been modified to report whether the EGs sampled are connected or not).\footnote{The modified program will be made available after publication.}\comments{\footnote{The original and the modified programs are available at www.ida.liu.se/$\sim$jospe.}} The program is written in C++ and compiled in Microsoft Visual C++ 2010 Express. The experiments are run on an AMD Athlon 64 X2 Dual Core Processor 5000+ 2.6 GHz, 4 GB RAM and Windows Vista Business. The compiler and the computer used in \cite{Penna2007} were Microsoft Visual C++ 2008 Express and a Pentium 2.4 GHz, 512 MB RAM and Windows 2000. The experimental settings is the same as before for up to 30 nodes, i.e. each approximate ratio reported is based on a sample of $10^4$ EGs, each obtained as the state of the MC after performing $10^6$ transitions with the empty EG as initial state. For 31 nodes though, each EG sampled is obtained as the state of the MC after performing $2 \times 10^6$ transitions with the empty EG as initial state. We elaborate later on why we double the length of the MCs for 31 nodes.

The rest of the paper is organized as follows. In Section \ref{sec:egs}, we extend the work in \cite{Penna2007} from 20 to 31 nodes. In Section \ref{sec:cegs}, we extend the work in \cite{Penna2007} with new approximate ratios. In Section \ref{sec:discussion}, we recall our findings and discuss future work. The paper ends with two appendices devoted to technical details.

\section{Extension from 20 to 31 Nodes}\label{sec:egs}

\begin{table}[t]
\caption{Exact and approximate $\frac{\#EGs}{\#DAGs}$ and $\frac{\#EDAGs}{\#EGs}$.}\label{tab:ratio} \scriptsize
\begin{center}
\renewcommand{\arraystretch}{1.2}
\begin{tabular}{c|rrr|rrr|rrr|}
\multicolumn{1}{c}{\bf NODES} & \multicolumn{3}{c}{\bf EXACT} & \multicolumn{3}{c}{\bf OLD APPROXIMATE} & \multicolumn{3}{c}{\bf NEW APPROXIMATE}\\
& $\frac{\#EGs}{\#DAGs}$ & $\frac{\#EDAGs}{\#EGs}$ & Hours & $\frac{\#EGs}{\#DAGs}$ & $\frac{\#EDAGs}{\#EGs}$ & Hours & $\frac{\#EGs}{\#DAGs}$ & $\frac{\#EDAGs}{\#EGs}$ & Hours\\
\hline 2 & 0.66667 & 0.50000 & 0.0 & 0.66007 & 0.50500 & 3.5 & 0.67654 & 0.49270 & 1.3\\
\hline 3 & 0.44000 & 0.36364 & 0.0 & 0.43704 & 0.36610 & 5.2 & 0.44705 & 0.35790 & 1.0\\
\hline 4 & 0.34070 & 0.31892 & 0.0 & 0.33913 & 0.32040 & 6.8 & 0.33671 & 0.32270 & 1.2\\
\hline 5 & 0.29992 & 0.29788 & 0.0 & 0.30132 & 0.29650 & 8.0 & 0.29544 & 0.30240 & 1.4\\
\hline 6 & 0.28238 & 0.28667 & 0.0 & 0.28118 & 0.28790 & 9.4 & 0.28206 & 0.28700 & 1.6\\
\hline 7 & 0.27443 & 0.28068 & 0.0 & 0.27228 & 0.28290 & 12.4 & 0.27777 & 0.27730 & 2.0\\
\hline 8 & 0.27068 & 0.27754 & 0.0 & 0.26984 & 0.27840 & 13.8 & 0.26677 & 0.28160 & 2.3\\
\hline 9 & 0.26888 & 0.27590 & 7.0 & 0.27124 & 0.27350 & 16.5 & 0.27124 & 0.27350 & 2.6\\
\hline 10 & 0.26799 & 0.27507 & 2253.0 & 0.26690 & 0.27620 & 18.8 & 0.26412 & 0.27910 & 3.1\\
\hline 11 & & & & 0.26179 & 0.28070 & 20.4 & 0.26179 & 0.28070 & 3.8\\
\hline 12 & & & & 0.26737 & 0.27440 & 21.9 & 0.26825 & 0.27350 & 4.2\\
\hline 13 & & & & 0.26098 & 0.28090 & 23.3 & 0.27405 & 0.26750 & 4.5\\
\hline 14 & & & & 0.26560 & 0.27590 & 25.3 & 0.27161 & 0.26980 & 5.1\\
\hline 15 & & & & 0.27125 & 0.27010 & 25.6 & 0.26250 & 0.27910 & 5.7\\
\hline 16 & & & & 0.25777 & 0.28420 & 27.3 & 0.26943 & 0.27190 & 6.7\\
\hline 17 & & & & 0.26667 & 0.27470 & 29.9 & 0.26942 & 0.27190 & 7.6\\
\hline 18 & & & & 0.25893 & 0.28290 & 37.4 & 0.27040 & 0.27090 & 8.2\\
\hline 19 & & & & 0.26901 & 0.27230 & 38.1 & 0.27130 & 0.27000 & 9.0\\
\hline 20 & & & & 0.27120 & 0.27010 & 40.3 & 0.26734 & 0.27400 & 9.9\\
\hline 21 & & & & & & & 0.26463 & 0.27680 & 17.4\\
\hline 22 & & & & & & & 0.27652 & 0.26490 & 18.8\\
\hline 23 & & & & & & & 0.26569 & 0.27570 & 13.3\\
\hline 24 & & & & & & & 0.27030 & 0.27100 & 14.0\\
\hline 25 & & & & & & & 0.26637 & 0.27500 & 15.9\\
\hline 26 & & & & & & & 0.26724 & 0.27410 & 17.0\\
\hline 27 & & & & & & & 0.26950 & 0.27180 & 18.6\\
\hline 28 & & & & & & & 0.27383 & 0.26750 & 20.1\\
\hline 29 & & & & & & & 0.27757 & 0.26390 & 21.1\\
\hline 30 & & & & & & & 0.28012 & 0.26150 & 21.6\\
\hline 31 & & & & & & & 0.27424 & 0.26710 & 47.3\\
\hline
\end{tabular}
\end{center}
\end{table}

Table \ref{tab:ratio} presents our new approximate ratios, together with the old approximate ones and the exact ones available in the literature. The first conclusion that we draw from the table is that
the new ratios are very close to the exact ones, as well as to the old ones. This makes
us confident on the accuracy of the ratios for 11-31
nodes, where no exact ratios are available in the literature due to
the high computational cost involved in calculating them. Another
conclusion that we draw from the table is that the ratios seem to be 0.26-0.28 for 11-31 nodes. This agrees well with the conjectured ratio of 0.27 for more than 10 nodes reported in \cite{GillispieandPerlman2002}. A last conclusion that we draw from the table is that the fraction of EGs that represent a unique DAG, i.e. $\frac{\#EDAGs}{\#EGs}$, is 0.26-0.28 for 11-31 nodes, a substantial fraction.

Recall from the previous section that we slightly modified the experimental setting for 31 nodes, namely we doubled the length of the MCs. The reason is as follows. We observed an increasing trend in $\frac{\#EGs}{\#DAGs}$ for 25-30 nodes, and interpreted this as an indication that we might be reaching the limits of our experimental setting. Therefore, we decided to double the length of the MCs for 31 nodes in order to see whether this broke the trend. As can be seen in Table \ref{tab:ratio}, it did. This suggests that approximating the ratio for more than 31 nodes will require larger MCs and/or samples than the ones used in this work.

Note that we can approximate the number of EGs for up to 31 nodes as $\frac{\#EGs}{\#DAGs} \#DAGs$, where $\frac{\#EGs}{\#DAGs}$ comes from Table \ref{tab:ratio} and $\#DAGs$ comes from \cite[Equation 8]{Robinson1977}. Alternatively, we can approximate it as $\frac{\#EGs}{\#EDAGs} \#EDAGs$, where $\frac{\#EGs}{\#EDAGs}$ comes from Table \ref{tab:ratio} and $\#EDAGs$ can be computed by \cite[p. 270]{Steinsky2003}.

Finally, a few words on the running times reported in Table \ref{tab:ratio} may be in place. First, note that the times reported in Table \ref{tab:ratio} for the exact ratios are borrowed from \cite{GillispieandPerlman2002} and, thus, they correspond to a computer program run on a "mid-1990s-era, midrange minicomputer". Therefore, a direct comparison to our times seems unadvisable. Second, our times are around four times faster than the old times. The reason may be in the use of a more powerful computer and/or a different version of the compiler. The reason cannot be in the difference in the computer programs run, since this is negligible. Third, the new times have some oddities, e.g. the time for two nodes is greater than the time for three nodes. The reason may be that the computer ran other programs while running the experiments reported in this paper.

\section{Extension with New Ratios}\label{sec:cegs}

\begin{table}[t]
\caption{Approximate $\frac{\#CEGs}{\#CDAGs}$, $\frac{\#CEGs}{\#EGs}$ and $\frac{\#CDAGs}{\#DAGs}$.}\label{tab:ratio2} \scriptsize
\begin{center}
\renewcommand{\arraystretch}{1.2}
\begin{tabular}{c|rrr|}
\multicolumn{1}{c}{\bf NODES} & \multicolumn{3}{c}{\bf NEW APPROXIMATE}\\
& $\frac{\#CEGs}{\#CDAGs}$ & $\frac{\#CEGs}{\#EGs}$ & $\frac{\#CDAGs}{\#DAGs}$\\
\hline 2 & 0.51482 & 0.50730 & 0.66667\\
\hline 3 & 0.39334 & 0.63350 & 0.72000\\ 
\hline 4 & 0.32295 & 0.78780 & 0.82136\\ 
\hline 5 & 0.29471 & 0.90040 & 0.90263\\ 
\hline 6 & 0.28033 & 0.94530 & 0.95115\\
\hline 7 & 0.27799 & 0.97680 & 0.97605\\
\hline 8 & 0.26688 & 0.98860 & 0.98821\\
\hline 9 & 0.27164 & 0.99560 & 0.99415\\
\hline 10 & 0.26413 & 0.99710 & 0.99708\\
\hline 11 & 0.26170 & 0.99820 & 0.99854\\
\hline 12 & 0.26829 & 0.99940 & 0.99927\\
\hline 13 & 0.27407 & 0.99970 & 0.99964\\
\hline 14 & 0.27163 & 0.99990 & 0.99982\\
\hline 15 & 0.26253 & 1.00000 & 0.99991\\
\hline 16 & 0.26941 & 0.99990 & 0.99995\\
\hline 17 & 0.26942 & 1.00000 & 0.99998\\ 
\hline 18 & 0.27041 & 1.00000 & 0.99999\\
\hline 19 & 0.27130 & 1.00000 & 0.99999\\
\hline 20 & 0.26734 & 1.00000 & 1.00000\\ 
\hline 21 & 0.26463 & 1.00000 & 1.00000\\
\hline 22 & 0.27652 & 1.00000 & 1.00000\\
\hline 23 & 0.26569 & 1.00000 & 1.00000\\
\hline 24 & 0.27030 & 1.00000 & 1.00000\\
\hline 25 & 0.26637 & 1.00000 & 1.00000\\
\hline 26 & 0.26724 & 1.00000 & 1.00000\\
\hline 27 & 0.26950 & 1.00000 & 1.00000\\
\hline 28 & 0.27383 & 1.00000 & 1.00000\\
\hline 29 & 0.27757 & 1.00000 & 1.00000\\
\hline 30 & 0.28012 & 1.00000 & 1.00000\\
\hline 31 & 0.27424 & 1.00000 & 1.00000\\
\hline $\infty$ & ? & ? & $\approx$ 1\\
\hline
\end{tabular}
\end{center}
\end{table}

In \cite[p. 153]{GillispieandPerlman2002}, it is stated that "the variables chosen for inclusion in a multivariate data set are not chosen at random but rather because they occur in a common real-world context, and hence are likely to be correlated to some degree". This implies that the EG learnt from some given data is likely to be connected. We agree with this observation, because we believe that humans are good at detecting sets of mutually uncorrelated variables so that the original learning problem can be divided into smaller independent learning problems, each of which results in a CEG. Therefore, although we still cannot say which EGs will be visited during the search, we can say that some of them will most likely be connected and some others disconnected. This raises the question of whether $\frac{\#CEGs}{\#CDAGs} \approx \frac{\#DEGs}{\#DDAGs}$ where DEGs and DDAGs stand for disconnected EGs and disconnected DAGs. In \cite[p. 154]{GillispieandPerlman2002}, it is also said that a consequence of the learnt EG being connected is "that a substantial number of undirected edges are likely to be present in the representative essential graph, which in turn makes it likely that the corresponding equivalence class size will be relatively large". In other words, they conjecture that the equivalence classes represented by CEGs are relatively large. We interpret the term "relatively large" as having a ratio smaller than $\frac{\#EGs}{\#DAGs}$. However, this conjecture does not seem to hold according to the approximate ratios presented in Table \ref{tab:ratio2}. There, we can see that $\frac{\#CEGs}{\#CDAGs} \approx$ 0.26-0.28 for 6-31 nodes and, thus, $\frac{\#CEGs}{\#CDAGs} \approx \frac{\#EGs}{\#DAGs}$. That the two ratios coincide is not by chance because $\frac{\#CEGs}{\#EGs} \approx$ 0.95-1 for 6-31 nodes, as can be seen in the table. A problem of this ratio being so close to 1 is that sampling a DEG is so unlikely that we cannot answer the question of whether $\frac{\#CEGs}{\#CDAGs} \approx \frac{\#DEGs}{\#DDAGs}$ with our sampling scheme. Therefore, we have to content with having learnt that $\frac{\#CEGs}{\#CDAGs} \approx \frac{\#EGs}{\#DAGs}$. It is worth mentioning that this result is somehow conjectured by Ko\v{c}ka when he states in a personal communication to Gillispie that "large equivalence classes are merely composed of independent classes of smaller sizes that combine to make a single larger class" \cite[p. 1411]{Gillispie2006}. Again, we interpret the term "large" as having a ratio smaller than $\frac{\#EGs}{\#DAGs}$. Again, we cannot check Ko\v{c}ka's conjecture because sampling a DEG is very unlikely. However, we believe that the conjecture holds, because we expect the ratios for those EGs with $k$ connected components to be around $0.27^k$, i.e. we expect the ratios of the components to be almost independent one of another. Gillispie goes on saying that "an equivalence class encountered at any single step of the iterative [learning] process, a step which may involve altering only a small number of edges (typically only one), might be quite small" \cite[p. 1411]{Gillispie2006}. Note that the equivalence classes that he suggests that are quite small must correspond to CEGs, because he suggested before that large equivalence classes correspond to DEGs. We interpret the term "quite small" as having a ratio greater than $\frac{\#EGs}{\#DAGs}$. Again, this conjecture does not seem to hold according to the approximate ratios presented in Table \ref{tab:ratio2}. There, we can see that $\frac{\#CEGs}{\#CDAGs} \approx$ 0.26-0.28 for 6-31 nodes and, thus, $\frac{\#CEGs}{\#CDAGs} \approx \frac{\#EGs}{\#DAGs}$.

From the results in Tables \ref{tab:ratio} and \ref{tab:ratio2}, it seems that the asymptotic values for $\frac{\#EGs}{\#DAGs}$, $\frac{\#EDAGs}{\#EGs}$, $\frac{\#CEGs}{\#CDAGs}$ and $\frac{\#CEGs}{\#EGs}$ should be around 0.27, 0.27, 0.27 and 1, respectively. It would be nice to have a formal proof of these results. In this paper, we have proven a related result, namely that the ratio of CDAGs to DAGs is asymptotically 1. The proof can be found in Appendix B. Note from Table \ref{tab:ratio2} that the asymptotic value is almost achieved for 6-7 nodes already. Our result adds to the list of similar results in the literature, e.g. the ratio of labeled connected graphs to labeled graphs is asymptotically 1 \cite[p. 205]{HararyandPalmer1973}.

Note that we can approximate the number of CEGs for up to 31 nodes as $\frac{\#CEGs}{\#EGs} \#EGs$, where $\frac{\#CEGs}{\#EGs}$ comes from Table \ref{tab:ratio2} and $\#EGs$ can be computed as shown in the previous section. Alternatively, we can approximate it as $\frac{\#CEGs}{\#CDAGs} \#CDAGs$, where $\frac{\#CEGs}{\#CDAGs}$ comes from Table \ref{tab:ratio2} and $\#CDAGs$ can be computed as shown in Appendix A.

Finally, note that the running times to obtain the results in Table \ref{tab:ratio2} are the same as those in Table \ref{tab:ratio}, because both tables are based on the same samples.

\section{Discussion}\label{sec:discussion}

In \cite{GillispieandPerlman2002}, it is shown that $\frac{\#EGs}{\#DAGs} \approx 0.27$ for 7-10 nodes. We have shown in this paper that $\frac{\#EGs}{\#DAGs} \approx$ 0.26-0.28 for 11-31 nodes. These results indicate that one should not expect more than a moderate gain in efficiency when searching the space of EGs instead of the space of DAGs. We have also shown that $\frac{\#CEGs}{\#CDAGs} \approx$ 0.26-0.28 for 6-31 nodes and, thus, $\frac{\#CEGs}{\#CDAGs} \approx \frac{\#EGs}{\#DAGs}$. Therefore, when searching the space of EGs, the fact that some of the EGs visited will most likely be connected does not seem to imply any additional gain in efficiency beyond that due to searching the space of EGs instead of the space of DAGs.

Some questions that remain open and that we would like to address in the future are checking whether $\frac{\#CEGs}{\#CDAGs} \approx \frac{\#DEGs}{\#DDAGs}$, and computing the asymptotic ratios of EGs to DAGs, EDAGs to EGs, CEGs to CDAGs, and of CEGs to EGs. Recall that in this paper we have proven that the asymptotic ratio of CDAGs to DAG is 1. Another topic for further research, already mentioned in \cite{Penna2007}, would be improving the graphical modifications that determine the MC transitions, because they rather often produce a graph that is not an EG. Specifically, the MC transitions are determined by choosing uniformly one out of seven modifications to perform on the current EG. Actually, one of the modifications leaves the current EG unchanged. Therefore, around 14 \% of the modifications cannot change the current EG and, thus, 86 \% of the modifications can change the current EG. In our experiments, however, only 6-8 \% of the modifications change the current EG. The rest up to the mentioned 86 \% produce a graph that is not an EG and, thus, they leave the current EG unchanged. This problem has been previously pointed out in \cite{Perlman2000}. Furthermore, he presents a set of more complex modifications that are claimed to alleviate the problem just described. Unfortunately, no evidence supporting this claim is provided. More recently, He et al. have proposed an alternative set of modifications having a series of desirable features that ensure that applying the modifications to an EG results in a different EG \cite{Heetal.2012}. Although these modifications are more complex than those in \cite{Penna2007}, the authors show that their MCMC approach is thousands of times faster for 3, 4 and 6 nodes \cite[pp. 17-18]{Heetal.2012}. However, they also mention that it is unfair to compare these two approaches: Whereas $10^4$ MCs of $10^6$ transitions each are run in \cite{Penna2007} to obtain a sample, they only run one MC of $10^4$-$10^5$ transitions. Therefore, it is not clear how their MCMC approach scales to 10-30 nodes as compared to the one in \cite{Penna2007}. The point of developing modifications that are more effective than ours at producing EGs is to make a better use of the running time by minimizing the number of graphs that have to be discarded. However, this improvement in effectiveness has to be weighed against the computational cost of the modifications, so that the MCMC approach still scales to the number of nodes of interest.

\section*{Appendix A: Counting CDAGs}

Let $A(x)$ denote the exponential generating function for DAGs. That is,
\[
A(x)=\sum_{k=1}^{\infty} \frac{A_k}{k!} x^k
\]
where $A_k$ denotes the number of DAGs of order $k$. Likewise, let $a(x)$ denote the exponential generating function for CDAGs. That is,
\[
a(x)=\sum_{k=1}^{\infty} \frac{a_k}{k!} x^k
\]
where $a_k$ denotes the number of CDAGs of order $k$. Note that $A_k$ can be computed without having to resort to enumeration by \cite[Equation 8]{Robinson1977}. However, we do not know of any formula to compute $a_k$ without enumeration. Luckily, $a_k$ can be computed from $A_k$ as follows. First, note that
\[
1+A(x)=e^{a(x)}
\]
as shown by \cite[pp. 8-9]{HararyandPalmer1973}. Now, let us define $A_0=1$ and redefine $A(x)$ as
\[
A(x)=\sum_{k=0}^{\infty} \frac{A_k}{k!} x^k,
\]
i.e. the summation starts with $k=0$. Then,
\[
A(x)=e^{a(x)}.
\]
Consequently,
\[
\frac{a_n}{n!}=\frac{A_n}{n!}-(\sum_{k=1}^{n-1} k \frac{a_k}{k!} \frac{A_{n-k}}{(n-k)!}) / n
\]
as shown by \cite[pp. 8-9]{HararyandPalmer1973}, and thus
\[
a_n=A_n-(\sum_{k=1}^{n-1} k \binom{n}{k} a_k A_{n-k}) / n.
\]
See also \cite[pp. 38-39]{Castelo2002}. Moreover, according to \cite[Sequence A082402]{oeis2010}, the result in this appendix has previously been reported in \cite{Robinson1973}. However, we could not gain access to that paper to confirm it.

\section*{Appendix B: Asymptotic Behavior of CDAGs}

\begin{thm}\label{the:ratio}
The ratio of CDAGs of order $n$ to DAGs of order $n$ tends to 1 as $n$ tends to infinity.
\end{thm}

\begin{proof}
Let $A_n$ and $a_n$ denote the numbers of DAGs and CDAGs of order $n$, respectively. Specifically, we prove that $(A_n/n!)/(a_n/n!) \rightarrow 1$ as $n \rightarrow \infty$. By \cite[Theorem 6]{Wright1967}, this holds if the following three conditions are met:
\begin{itemize}
\item[(i)] $\log ((A_n/n!)/(A_{n-1}/(n-1)!)) \rightarrow \infty$ as $n \rightarrow \infty$,

\item[(ii)] $\log ((A_{n+1}/(n+1)!)/(A_n/n!)) \geq \log ((A_n/n!)/(A_{n-1}/(n-1)!))$ for all large enough $n$, and

\item[(iii)] $\sum_{k=1}^{\infty} (A_k/k!)^2/(A_{2k}/(2k)!)$ converges.
\end{itemize}

We start by proving that the condition (i) is met. Note that from every DAG $G$ over the nodes $\{v_1, \ldots, v_{n-1}\}$ we can construct $2^{n-1}$ different DAGs $H$ over $\{v_1, \ldots, v_n\}$ as follows: Copy all the arrows from $G$ to $H$ and make $v_n$ a child in $H$ of each of the $2^{n-1}$ subsets of $\{v_1, \ldots, v_{n-1}\}$. Therefore,
\[
\log ((A_n/n!)/(A_{n-1}/(n-1)!)) \geq \log (2^{n-1}/n)
\]
which clearly tends to infinity as $n$ tends to infinity.

We continue by proving that the condition (ii) is met. Every DAG over the nodes $V \cup \{w\}$ can be constructed from a DAG $G$ over $V$ by adding the node $w$ to $G$ and making it a child of a subset $Pa$ of $V$. If a DAG can be so constructed from several DAGs, we simply consider it as constructed from one of them. Let $H_1, \ldots, H_m$ represent all the DAGs so constructed from $G$. Moreover, let $Pa_i$ denote the subset of $V$ used to construct $H_i$ from $G$. From each $Pa_i$, we can now construct $2m$ DAGs over $V \cup \{w, u\}$ as follows: (i) Add the node $u$ to $H_i$ and make it a child of each subset $Pa_j \cup \{w\}$ with $1 \leq j \leq m$, and (ii) add the node $u$ to $H_i$ and make it a parent of each subset $Pa_j \cup \{w\}$ with $1 \leq j \leq m$. Therefore, $A_{n+1}/A_n \geq 2 A_n/A_{n-1}$ and thus
\[
\log ((A_{n+1}/(n+1)!)/(A_n/n!)) = \log (A_{n+1}/A_n) - \log (n+1)
\]
\[
\geq \log (2 A_n/A_{n-1}) - \log (n+1)
\geq \log (2 A_n/A_{n-1}) - \log (2 n) = \log (A_n/A_{n-1}) - \log n 
\]
\[
= \log ((A_n/n!)/(A_{n-1}/(n-1)!)).
\]

Finally, we prove that the condition (iii) is met. Let $G$ and $G'$ denote two (not necessarily distinct) DAGs of order $k$. Let $V=\{v_1, \ldots, v_k\}$ and $V'=\{v'_1, \ldots, v'_k\}$ denote the nodes in $G$ and $G'$, respectively. Consider the DAG $H$ over $V \cup V'$ that has the union of the arrows in $G$ and $G'$. Let $w$ and $w'$ denote two nodes in $V$ and $V'$, respectively. Let $S$ be a subset of size $k-1$ of $V \cup V' \setminus \{w, w'\}$. Now, make $w$ a parent in $H$ of all the nodes in $S \cap V'$, and make $w'$ a child in $H$ of all the nodes in $S \cap V$. Note that the resulting $H$ is a DAG of order $2k$. Note that there are $k^2$ different pairs of nodes $w$ and $w'$. Note that there are $\binom{2k-2}{k-1}$ different subsets of size $k-1$ of $V \cup V' \setminus \{w, w'\}$. Note that every choice of DAGs $G$ and $G'$, nodes $w$ and $w'$, and subset $S$ gives rise to a different DAG $H$. Therefore, $A_{2k}/A_k^2 \geq k^2 \binom{2k-2}{k-1}$ and thus
\[
\sum_{k=1}^{\infty} (A_k/k!)^2/(A_{2k}/(2k)!)
= \sum_{k=1}^{\infty} A_k^2 (2k)!/(A_{2k} k!^2)
\]
\[
\leq \sum_{k=1}^{\infty} ((k-1)! (k-1)! (2k)!)/(k^2 (2k-2)! k!^2)
= \sum_{k=1}^{\infty} (4k-2)/k^3
\]
which clearly converges.
\end{proof}

\subsubsection*{Acknowledgments.}

This work is funded by the Center for Industrial Information Technology (CENIIT) and a so-called career contract at Link\"oping University, by the Swedish Research Council (ref. 2010-4808), and by FEDER funds and the Spanish Government (MICINN) through the project TIN2010-20900-C04-03. We thank Dag Sonntag for his comments on this work.

\end{document}